\documentclass{article}
\usepackage[utf8]{inputenc} %
\usepackage[T1]{fontenc}    %
\usepackage{amsfonts}       %
\usepackage{amsmath}
\usepackage{amssymb}
\usepackage{mathtools}
\usepackage{amsthm}
\usepackage{microtype}
\usepackage{bm}
\usepackage{color}
\usepackage{soul}           %
\usepackage{url}
\usepackage{booktabs}       %
\usepackage{nicefrac}       %
\usepackage{enumerate}
\usepackage{graphicx}
\DeclareGraphicsExtensions{.pdf,.png,.jpg}
\graphicspath{{figures/}}
\usepackage{natbib}
\usepackage{comment}
\usepackage{hyperref}
\usepackage{centernot}     %
\usepackage{algorithm}
\usepackage{algpseudocode}
\usepackage{fullpage}

\newcommand{\R}{\mathbb{R}}

\newcommand{\BigO}[1]{\ensuremath{\mathcal{O}\left(#1\right)}}                  %
\newcommand{\BigOT}[1]{\ensuremath{\widetilde{\mathcal{O}}\left(#1\right)}}     %
\newcommand{\Exp}[2][]{\ensuremath{\mathbb{E}_{#1}\left[#2\right]}}                %
\newcommand{\Ind}[1]{\ensuremath{\mathbf{1}\left[#1\right]}}                     %
\newcommand{\NormI}[1]{\ensuremath{\left\lVert #1 \right\rVert}_1}               %
\newcommand{\parg}{\makebox[1ex]{$\mathbf{\cdot}$}}                              %
\newcommand{\mquad}{\kern-1em}													%
\newcommand{\Ineq}[2][]{\overset{#1}{#2}}                                         %

\newcommand{\InNorm}[1]{{\left\vert\kern-0.2ex\left\vert\kern-0.2ex\left\vert #1 
    \right\vert\kern-0.2ex\right\vert\kern-0.2ex\right\vert}}                    %

\newcommand{\InNormII}[1]{{\left\vert\kern-0.2ex\left\vert\kern-0.2ex\left\vert #1 
    \right\vert\kern-0.2ex\right\vert\kern-0.2ex\right\vert}_2}                    %

\newcommand{\InNormInfty}[1]{{\left\vert\kern-0.2ex\left\vert\kern-0.2ex\left\vert #1 
    \right\vert\kern-0.2ex\right\vert\kern-0.2ex\right\vert}_{\infty}}           %

\newcommand{\Abs}[1]{\ensuremath{\left \lvert #1 \right \rvert}}                 %
\newcommand{\Prob}[2][]{\ensuremath{\mathrm{Pr}_{#1}\left\{ #2 \right\}}}        %
\newcommand{\iid}{i.i.d.~}                                                        %
\newcommand{\simiid}{\overset{\mathrm{i.i.d.}}{\sim}}                            %
\newcommand{\Grad}{\nabla}                                                       %
\DeclarePairedDelimiterX{\Inner}[2]{\langle}{\rangle}{#1, #2}                    %
\newcommand{\Land}{\wedge}                                                       %
\newcommand{\defeq}{\overset{\mathrm{def}}{=}}                                                      %

\newcommand{\Set}[1]{\{#1\}}                                                     %
\DeclareMathOperator*{\union}{\cup}

\definecolor{gray}{rgb}{0.7,0.7,0.7}

\DeclareMathOperator*{\argmax}{argmax}

\newtheorem{definition}{Definition}

\newtheorem{assumption}{Assumption}
\newtheorem{lemma}{Lemma}
\newtheorem{theorem}{Theorem}
\newtheorem{remark}{Remark}

\theoremstyle{definition}

\newcommand{\dX}{\mathfrak{X}}   %
\newcommand{\dY}{\mathfrak{Y}}   %
\newcommand{\dW}{\R^{d,s}}   %
\newcommand{\dWp}{\mathcal{W}}  %
\newcommand{\sG}{\mathfrak{G}}   %
\newcommand{\wmin}{w_{\mathrm{min}}}

\newcommand{\eRC}{\widehat{\mathfrak{R}}}		   %

\newcommand{\sS}{\mathsf{S}}    %
\newcommand{\sT}{\mathsf{T}}    %
\newcommand{\sTa}{\bar{\mathsf{T}}}    %

\newcommand{\pD}{\mathcal{D}}    %
\newcommand{\pG}{\mathcal{G}}    %
\newcommand{\pR}{\mathcal{R}}	    %
\newcommand{\pQ}{\mathcal{Q}}    

\allowdisplaybreaks      %

\newcommand{\CRFRAND}{\texttt{CRF\_RAND}}
\newcommand{\CRFALL}{\texttt{CRF\_ALL}}
\newcommand{\SVMRAND}{\texttt{SVM\_RAND}}
\newcommand{\SVM}{\texttt{SVM}}

\title{Learning Maximum-A-Posteriori Perturbation Models for Structured Prediction in Polynomial Time}
\author{Asish Ghoshal and Jean Honorio\\
Department of Computer Science\\
Purdue University\\
West Lafayette, IN - 47906\\
\{aghoshal, jhonorio\}@purdue.edu}

\date{}

\begin{document}
\maketitle

\begin{abstract}
MAP perturbation models have emerged as a powerful framework for inference in structured prediction. Such models
provide a way to efficiently sample from the Gibbs distribution and facilitate predictions that are
robust to random noise. In this paper, we propose a provably polynomial time randomized algorithm
for learning the parameters of perturbed MAP predictors. Our approach is based on minimizing a novel 
Rademacher-based generalization bound on the expected loss of a perturbed MAP predictor, which can be computed in polynomial time.
We obtain conditions under which our randomized learning algorithm can guarantee generalization to unseen examples.
\end{abstract}

\section{Introduction}
Structured prediction can be thought of as a
generalization of binary classification to structured outputs,
where the goal is to jointly predict several dependent variables.
Predicting complex, structured data is of great significance in
various application domains including computer vision (e.g.,
image segmentation, multiple object tracking), natural language processing
(e.g., part-of-speech tagging, named entity recognition) and computational biology 
(e.g. protein structure prediction). However, unlike binary classification,
structured prediction presents a set of unique computational and statistical challenges.
The chief being that the number of structured outputs is exponential in the input size.
For instance, in translation tasks, the number of parse trees of a sentence is exponential
in the length of the sentence. Second, it is very common in such domains to have very few
training examples as compared to the size of the output space thereby making generalization
to unseen inputs difficult.

The key computational challenge in structured prediction stems from the \emph{inference}
problem, where a \emph{decoder}, parameterized by a vector $w$ of weights, predicts (or \emph{decodes})
the latent structured output $y$ given an observed input $x$. With the exception of a few special
cases, the general inference problem in structured prediction is intractable.
For instance in many cases the inference problem reduces to the maximum acyclic subgraph problem
which is NP-hard and hard to approximate to within a factor of $\nicefrac{1}{2}$ of the optimal solution
\cite{guruswami2008beating}, or cardinality-constrained submodular maximization, 
which is also NP-hard and hard to compute a solution better than the 
$(1 - \nicefrac{1}{\varepsilon})$-approximate 
solution  returned by a greedy algorithm \cite{nemhauser1978analysis}. 
The \emph{learning} problem, where the goal is to learn the
parameter $w$ of the decoder from a set of labeled training instances, and which involves solving
the inference problem as a subroutine, is therefore intractable for all but a few special cases.

Hardness results notwithstanding, various methods --- which are worst-case exponential-time ---
have been developed  over the last decade for predicting structured
data including conditional random fields \cite{lafferty2001conditional},
and max-margin approaches \cite{taskar_max-margin_2003}, to name a few. In these approaches,
learning the parameter $w$ of the decoder involves minimizing a loss function $L(w,\sS)$ over
a data set $\sS$ of $m$ training pairs $\Set{(x_i, y_i)}_{i=1}^m$. One could also take a Bayesian approach
and learn a posterior distribution $\pQ$ over decoder parameters $w$ by minimizing the Gibbs loss
$\Exp[w \sim \pQ]{L(w, \sS)}$. McAllester \cite{mcallester_generalization_2007} showed, using the PAC-Bayesian framework, 
that the commonly used max-margin loss \cite{taskar_max-margin_2003} upper bounds the expected Gibbs loss over the data distribution,
upto statistical error. Therefore, minimizing the max-margin loss
provides a principled way for learning the parameters of a structured decoder. More recently,
\cite{honorio_structured_2015} showed that minimizing a \emph{surrogate} randomized loss,
where the max-margin loss is computed over a small number of randomly sampled structured outputs,
also bounds the Gibbs loss from above upto statistical error. 

The above can be thought of as weight based perturbation models. The perturb-and-MAP framework
introduced by \cite{papandreou2011perturb}, and henceforth referred to as MAP perturbation, provides
an efficient way to generate samples from the Gibbs distribution by injecting
random noise  (that do not depend on the weights of the decoder $w$) in the potential 
or score function of the decoder and then computing the most likely assignment or energy configuration (MAP).
MAP perturbation models are an attractive alternative to expensive Markov Chain Monte Carlo simulations for drawing
samples from the Gibbs distribution, in that the former facilitates one-shot sampling. 
Moreover, learning MAP predictors for structured prediction
problems is particularly attractive because the predictions are robust to random noise. However, learning the
parameters of such MAP predictors involves solving the MAP problem, which in general is intractable. In this
paper we obtain a provably polynomial time algorithm for learning the parameters of perturbed MAP predictors
with structure based perturbations.
In the following paragraph we summarize the main technical contributions of our paper.

\paragraph{Our contributions.}
To the best of our knowledge, we are the first to obtain generalization bounds for MAP-perturbation
models with structure-based (Gumbel) perturbations 
--- for detailed comparison with existing literature see Section \ref{sec:related_work}.
While it is well known that Gumbel perturbations induce a conditional random field (CRF)
distribution over the structured outputs, we show that the generalization error is upper bounded
by a CRF loss up to statistical error. We obtain Rademacher based uniform convergence guarantees for the latter. However, the
main contribution of our paper is to obtain a provably polynomial time algorithm for learning MAP-perturbation models
for general structured prediction problems.
We propose a novel randomized \emph{surrogate loss} that lower bounds the CRF loss and still upper bounds
the expected loss over data distribution, upto approximation and statistical error terms that decay as $\BigOT{\nicefrac{1}{\sqrt{m}}}$
with $m$ being the number of samples. While it is NP-Hard to compute and approximate the CRF loss in general 
\cite{barahona1982computational, chandrasekaran2008complexity},
our surrogate loss can be computed in polynomial time. Our results also imply that one can learn parameters 
of CRF models for structured prediction in polynomial time under certain conditions.
Our work is inspired by the work of \cite{honorio_structured_2015} who also
propose a polynomial time algorithm for learning the parameters of a structured decoder in the max-margin framework. 
In contrast to prior work which consider weight based perturbations, our work 
is concerned with structure based perturbations. Previous algorithms for learning MAP perturbation models,
for instance, the hard-EM algorithm by \cite{gane2014learning} and the moment-matching algorithm by \cite{papandreou2011perturb},
are in general intractable and have no generalization guarantees. \emph{Lastly, the
main conceptual contribution of our work is to demonstrate that
it is possible to efficiently learn the parameters of a structured decoder with generalization guarantees without
solving the inference problem exactly.}

\section{Preliminaries}
We begin this section by introducing our notations and formalizing the problem
of learning MAP-perturbation models. In structured prediction, we have an input
$x \in \dX$ and a set of feasible decodings of the input $\dY(x)$. 
Without loss of generality, we assume that $\Abs{\dY(x)} \leq r$ for all $x \in \dX$.
Input-output pairs $(x,y)$ are represented by a joint feature vector $\phi(x,y) \in \R^d$.
For instance, when $x$ is a sentence and $y$ is a parse tree, the joint feature map $\phi(x,y)$
can be a vector of $0/1$-indicator variables representing if a particular word is
present in $x$ and a particular edge is present in $y$. We will assume that $\min \Set{\phi_j(x,y) \neq 0 \mid j \in [d]} \geq 1$
which commonly holds for structured prediction problems, for instance, when using binary features, or features that ``count'' 
number of components, edges, parts, etc.

A decoder $f_w: \dX \rightarrow \dY$, parameterized by a vector $w \in \R^d$, returns an
output $y \in \dY(x)$ given an input $x$. We consider linear decoders of the form:
\begin{equation}
f_w(x) = \argmax_{y \in \dY(x)} \; \Inner{\phi(x,y)}{w}, \label{eq:decoder}
\end{equation}
which return the highest scoring structured output for a particular input $x$, where
the score is linear in the weights $w$. As is traditionally the case in high-dimensional
statistics, we will assume that the weight vectors are $s$-sparse, i.e., have at most $s$ non-zero
coordinates. We will denote the set of $s$-sparse $d$-dimensional vectors by $\dW$.

In the perturb and MAP framework, a \emph{stochastic decoder} 
first perturbs the linear score by injecting some independent noise for each structured output $y$,
and then returns the structured output that maximizes the perturbed score. Gumbel perturbations are
commonly used owing to the max-stability property of the Gumbel distribution.
Denoting $\pG(\beta)$ as the Gumbel distribution with location and scale parameters 0 and $\beta$ respectively,
we have the following stochastic decoder, where $\gamma \sim \pG^r$ denotes 
a collection of $r$ \iid Gumbel-distributed random variables and $\gamma_y$ denotes the Gumbel
random variable associated with structured output $y$:
\begin{align}
f_{w,\gamma}(x) = \argmax_{y \in \dY(x)} \Inner{\phi(x,y)}{w} + \gamma_y. \label{eq:stochastic_decoder}
\end{align}
For any weight vector $w$, and data set $\sS = \Set{(x_i, y_i)} \simiid \pD^m$, 
we consider the following expected and empirical zero-one loss:
\begin{align}
L(w, \pD) &= \Exp[(x,y) \sim \pD]{\Exp[\gamma \sim \pG^r]{\Ind{y \neq f_{w, \gamma}(x)}}}, \label{eq:exp_loss} \\
L(w, \sS) &= \frac{1}{m} \sum_{i=1}^m \Exp[\gamma \sim \pG^r]{\Ind{y_i \neq f_{w, \gamma}(x_i)}} \label{eq:emp_loss},
\end{align}
where $\Ind{\parg}$ denotes the indicator function and $\pD$ is the unknown data distribution.
We will let the scale parameter depend on the number of samples $m$ and the weight vector $w$, and write $\beta(m, w) > 0$.
The reason for this will become clear later, but intuitively one would expect that as the number of samples
increases, the magnitude of perturbations should decrease in order to control the generalization error.
Under Gumbel perturbations, $f_{w,\gamma}(x_i)$ is distributed according to following conditional
random field (CRF) distribution $\pQ(x_i, w)$ with pmf $q(\parg; x_i, w)$ \cite{gumbel1954statistical, papandreou2011perturb}:
\begin{align}
q(y_i; x_i, w) &= \Prob[\gamma \sim \pG^r(\beta)]{f_{w, \gamma}(x_i) = y_i}  
= \frac{\exp(\nicefrac{\Inner{\phi(x_i,y_i)}{w}}{\beta})}{Z(w, x_i)} ,
 \label{eq:full_crf}
\end{align}
where $Z(w, x_i) = \sum_{y \in \dY(x)} \exp(\nicefrac{\Inner{\phi(x_i,y)}{w}}{\beta}) $ is the partition function. 
The empirical loss in \eqref{eq:emp_loss} can then be computed as:
\begin{align}
L(w, \sS) = \frac{1}{m} \sum_{i=1}^m \Prob{f_{w, \gamma}(x_i) \neq y_i} && (\text{\textbf{CRF loss}}) \label{eq:full_loss}.
\end{align}

The ultimate objective of a learning algorithm is to learn a weight vector $w$
that generalizes to unseen data. Therefore, minimizing the expected loss given by \eqref{eq:exp_loss}
is the best strategy towards that end. However, since the data distribution is unknown, one instead minimizes
the empirical loss \eqref{eq:emp_loss} on a finite number of labeled examples $\sS$.

\section{Generalization Bound}
As a first step we will show that the empirical loss \eqref{eq:full_loss} indeed bounds the
expected perturbed loss \eqref{eq:exp_loss} from above, upto statistical error that decays as $\BigOT{\nicefrac{1}{\sqrt{m}}}$. 
We have the following generalization bound.
\begin{theorem}[Rademacher based generalization bound]
\label{thm:gen_bound}
With probability at least $1 - \delta$ over the choice of $m$ samples $\sS$: 
\begin{align*}
(\forall w \in \dW) ~ L(w, \pD) \leq L(w, \sS) + \varepsilon(d,s,m,r,\delta),
\end{align*}
where
\begin{align*}
\varepsilon(d,s,m,r,\delta) = 2 \sqrt{\frac{s (\ln d + 2 \ln (mr))}{m}} + 3 \sqrt{\frac{\ln \nicefrac{2}{\delta}}{2m}}.
\end{align*}
\end{theorem}

\begin{proof}
Let 
\begin{align*}
g_w(x,y) &\defeq \Prob[\gamma \sim \pG^r(\beta)]{y \neq f_{w, \gamma}(x)}, \\
\sG &\defeq \Set{g_w \mid w \in \dW }.
\end{align*}
Then by Rademacher based
uniform convergence, with probability at least $1 - \delta$ over the choice
of $m$ samples, we have that:
\begin{align}
(\forall w \in \dW) ~ L(w, \pD) \leq L(w, \sS) + 2 \eRC_{\sS}(\sG) + 3 \sqrt{\frac{\log \nicefrac{2}{\delta}}{2m}},
	\label{eq:gen_bound}
\end{align}
where $\eRC_{\sS}(\sG)$ denotes the empirical Rademacher complexity of $\sG$.
Let $\sigma = (\sigma_i)_{i=1}^m$ be independent Rademacher variables. Also
define $\dWp \defeq \Set{\nicefrac{w}{\beta(w, m)} \mid w \in \R^{d,s}}$. Then, 
\begin{align*}
\eRC_{\sS}(\sG)  
&= \Exp[\sigma]{\sup_{w \in \dW} \frac{1}{m} \sum_{i=1}^m \sigma_i g_w(x_i,y_i)} \\
&= \frac{1}{m} \Exp[\sigma]{\sup_{w \in \dW}  \sum_{i=1}^m  \sigma_i \Prob[\gamma \sim \pG^r(\beta)]{y_i \neq f_{w, \gamma}(x_i)}} \\
&\Ineq[(a)]{=} \frac{1}{m} \Exp[\sigma]{\sup_{w \in \dWp}  \sum_{i=1}^m \sigma_i \Prob[\gamma \sim \pG^r(1)]{y_i \neq f_{w, \gamma}(x_i)}} \\
&\leq \frac{1}{m} \Exp[\gamma \sim \pG^r(1)]{\Exp[\sigma]{\sup_{w \in \dWp} \sum_{i=1}^m \sigma_i \Ind{y_i \neq f_{w, \gamma}(x_i)}}} \\
&\Ineq[(b)]{\leq} \frac{1}{m} \Exp[\gamma \sim \pG^r(1)]{\Exp[\sigma]{\sup_{w \in \dW} 
	\sum_{i=1}^m \sigma_i \Ind{y_i \neq f_{w, \gamma}(x_i)}}},
\end{align*}
where step (a) follows from 
$\Prob[\gamma \sim \pG^r(\beta)]{y_i \neq f_{w, \gamma}(x_i)}$ $= \Prob[\gamma \sim \pG^r(1)]{y_i \neq f_{\nicefrac{w}{\beta}, \gamma}(x_i)}$,
and step (b) follows from $\dWp \subseteq \dW$.
We will enumerate the structured outputs $\dY(x_i)$ as $y_{i,1}, \ldots, y_{i,r}$.
For any fixed $\gamma$, the weight vector $w$ induces a linear ordering $\pi_i(\parg ; \gamma)$ 
over the structured outputs $\dY(x_i)$, i.e., 
$\Inner{\phi(x_i, y_{i,\pi_i(1; \gamma)})}{w} + \gamma_1 > \Inner{\phi(x_i, y_{i,\pi_i(2; \gamma)})}{w} + \gamma_2 
> \ldots > \Inner{\phi(x_i, y_{i,\pi_i(r; \gamma)})}{w} + \gamma_r.$
Let $\pi(\gamma) = \Set{\pi_i}$ be the orderings over all $m$ data points induced by a fixed weight vector $w$ and fixed $\gamma$,
and let $\Pi(\gamma)$ be the collection of all orderings $\pi(\gamma)$ over all $w \in \dW$ for a fixed $\gamma$. 
Since $w$ is $s$-sparse we have, from results by \cite{bennett1956determination,
bennett1960multidimensional, cover1967number},
that the number of possible linear orderings is 
$\Abs{\Pi(\gamma)} \leq {d \choose s} (mr)^{2s} \leq d^s (mr)^{2s}$ . Therefore we have:
\begin{align*}
\eRC_{\sS}(\sG) 
	&\leq \frac{1}{m} \Exp[\gamma \sim \pG^r(\beta)]{\! \Exp[\sigma]{\sup_{\pi(\gamma) \in \Pi(\gamma)} 
		\! \sum_{i=1}^m \! \sigma_i \Ind{y_i \neq y_{i,\pi_i(1; \gamma)}}}} \\
	& \Ineq[(a)]{\leq} \frac{1}{m} \sqrt{s (\log d + 2 \log (mr))} \sqrt{m} \\
	& = \sqrt{\frac{s (\log d + 2 \log (mr))}{m}},
\end{align*}
where the inequality (a) follows from the Massart's finite class lemma.
\end{proof}
As a direct consequence of the
uniform convergence bound given by Theorem \ref{thm:gen_bound}, we have that minimizing the CRF loss
\eqref{eq:full_loss} is a consistent procedure for learning MAP-perturbation models.

\section{Towards an efficient learning algorithm}
While Theorem \ref{thm:gen_bound} provides theoretical justification for learning MAP-perturbation models
by minimizing the CRF loss \eqref{eq:full_loss}, with the exception of a few special cases, computing the loss function 
is in general intractable. This is due to the need for computing the partition function $Z(w,x)$ which is
an NP-hard problem \cite{barahona1982computational}. Further, even approximating
$Z(w,x)$ with high probability and arbitrary precision is also known to be NP-hard \cite{chandrasekaran2008complexity}.

To counter this computational bottleneck, we propose an efficient stochastic decoder that decodes over
a randomly sampled set of structured outputs. To elaborate further, given some $x \in \dX$, let $\pR(x,w)$ be 
some \emph{proposal distribution}, parameterized by $x$ and $w$, 
over the structured outputs $\dY(x)$.
We generate a set $\sT'$ of $n$ structured outputs sampled independently from the distribution $\pR$ and define
the following \emph{efficient} stochastic decoder:
\begin{align}
f_{w, \gamma, \sT'}(x) = \argmax_{y \in \sT'} \Inner{\phi(x,y)}{w} + \gamma_y.
\end{align}
Therefore $f_{w, \gamma, \sT'}(x)$ is distributed according to the CRF distribution $\pQ(x, w, \sT')$
with pmf $q(\parg; x, w, \sT')$ and support on $\sT'$ as follows:
\begin{align*}
q(y; x, w, \sT') &= \Prob[\gamma \sim \pG^n]{f_{w, \gamma, \sT'}(x) = y} 
	= \frac{\Ind{y \in \sT'}}{Z_{w,x,\sT'}} \exp(\nicefrac{\Inner{\phi(x,y)}{w}}{\beta}),
\end{align*}
where $Z_{w,x,\sT'} = \sum_{y' \in \sT'} \exp(\nicefrac{\Inner{\phi(x,y')}{w}}{\beta})$. 
Note that the partition function $Z_{w,x,\sT'}$ can be computed in time linear in $n$, 
since $\Abs{\sT'} = n$. Now, let $\sT = \Set{\sT_i \mid x_i \in \sS}$ be the collection
of $n$ structured outputs sampled for each $x_i$ in the data set, 
from the product distribution $\pR(\sS,w) \defeq \times_{i=1}^m (\pR(x_i)^{n})$.
\emph{Note that the distribution $\pR(\sS,w)$ does not depend on the $\Set{y_i}$'s in $\sS$. 
We denote the distribution over the collection of sets $\Set{\sT_i}$ by $\pR(\sS,w)$ to keep
the notation light}. Additionally, we consider proposal distributions $\pR(x, w)$ that are 
equivalent upto linearly inducible orderings of the structured output.
\begin{definition}[Equivalence of proposal distributions \cite{honorio_structured_2015}]
\label{def:equivalent_proposals}
For any $x \in \dX$, two proposal distributions $\pR(x,w)$ and $\pR(x,w')$,
with probability mass functions  $p(\parg; x,w)$ and $p(\parg; x, w')$,
are equivalent if: 
\begin{gather*}
\forall y, y' \in \dY(x)\;: \Inner{\phi(x,y)}{w} \leq \Inner{\phi(x,y')}{w} \\
\text{ and }  \Inner{\phi(x,y)}{w'} \leq \Inner{\phi(x,y')}{w'} \\
\iff \forall y \in \dY(x)\; p(y; x, w) = p(y; x, w').
\end{gather*}
We then write $\pR(x,w) \equiv \pR(x,w') \equiv \pR(x,\pi(x))$, where
$\pi(x)$ is the linear ordering over $\dY(x)$ induced by $w$ (and $w'$).
\end{definition}

Intuitively speaking, the above definition requires proposal distributions to depend only on
the orderings of the values $\Inner{\phi(x,y_1)}{w}, \ldots, \Inner{\phi(x,y_r)}{w}$ and not
on the actual value of $\Inner{\phi(x,y_j)}{w}$.

To obtain an efficient learning algorithm with generalization guarantees, we will use
\emph{augmented} sets $\sTa = \Set{\sTa_i}_{i=1}^m$, where $\sTa_i = \sT_i \union \Set{y_i}$.
Then, given a random collection of structured outputs $\sT$, we consider the following \emph{augmented randomized}
empirical loss for learning the parameters of the MAP-perturbation model:
\begin{equation}
L(w, \sS, \sTa) = \frac{1}{m} \sum_{i = 1}^m \Prob[\gamma \sim \pG^n]{f_{w, \gamma, \sTa_i}(x_i) \neq y_i}.
\label{eq:rand_loss}
\end{equation}
As opposed to the loss function given by \eqref{eq:full_loss}, the loss in \eqref{eq:rand_loss} can be computed efficiently for small $n$. 
Our next result shows that the randomized augmented loss lower bounds the full CRF loss $L(w,\sS)$ 
as long as $\sTa_i$ is a \emph{set}, i.e., contains only unique elements.
\begin{lemma}
\label{lemma:loss_lb}
For all data sets $\sS$, $\sT_i \subseteq \dY(x_i)$, and weight vectors $w$:
\begin{align}
L(w,\sS,\sTa) - L(w,\sS) = 
	 - \frac{1}{m} \sum_{i=1}^m \Prob[\gamma]{f_{w,\gamma,\sTa_i}(x_i) = y_i} 
		\Prob[\gamma]{f_{w,\gamma}(x_i) \in (\dY(x_i) \setminus \sTa_i)} 
	 \leq 0
\end{align}
\end{lemma}
\begin{proof}
For any $x \in \dX$, $\sT \subseteq \dY(x)$, $y \in \sT$ and weight vector $w$:
\begin{align*}
&\Prob[\gamma]{f_{w,\gamma}(x) = y} - \Prob[\gamma]{f_{w,\gamma,\sT}(x) = y} \\
&= e^{\Inner{\phi(x,y)}{w}} \left\{ \frac{Z(w, x, \sT) - Z(w,x)}{Z(w,x) Z(w, x, \sT)} \right\}  \\
&= \frac{e^{\Inner{\phi(x,y)}{w}}}{Z(w,x,\sT)}  \frac{1}{Z(w,x)} 
	\left\{-\!\! \sum_{y' \in \dY(x) \setminus \sT} \mquad e^{\Inner{\phi(x,y')}{w}} \right\} \\
&= -\Prob[\gamma]{f_{w,\gamma,\sT}(x) = y} \Prob[\gamma]{f_{w,\gamma}(x) \in \dY(x) \setminus \sT}.	
\end{align*}
Since by construction $y_i \in \sTa_i$, the final claim follows.
\end{proof}
\begin{remark}
If $\sTa_i = \dY(x_i)$ then $L(w, \sS) = L(w, \sS, \sTa_i)$.
\end{remark}

Next, we will show that an algorithm that learns the parameter $w$ of the MAP-perturbation model, 
by sampling a small number of structured outputs for each $x_i$ and minimizing
the empirical loss given by \eqref{eq:rand_loss}, generalizes under various choices of the proposal distribution $\pR$.
Our first step in that direction would be to obtain uniform convergence guarantees for the stochastic loss \eqref{eq:rand_loss}.

\subsection{Generalization bound}
To obtain our generalization bound, we decompose the
difference $L(w, \sS) - L(w, \sS, \sTa)$ as follows:
\begin{gather}
L(w, \sS) - L(w, \sS, \sTa) = A(w, \sS) + B(w, \sS, \sTa), \\
A(w, \sS) = L(w, \sS) - \Exp[\sT \sim \pR(\sS)]{L(w, \sS, \sTa)}, \\
B(w, \sS, \sTa) = \Exp[\sT \sim \pR(\sS)]{L(w, \sS, \sTa)} - L(w, \sS, \sTa),
\end{gather}
where $A(w, \sS)$ can be thought of as the approximation error due to using a small number of structured
outputs $\sT_i$'s instead of the full sets $\dY(x_i)$, while $B(w, \sS, \sTa)$ be is the statistical error. 
In what follows, we will bound each of these errors from above.

From Lemma \ref{lemma:loss_lb} it is clear that the proposal distribution plays a crucial
role in determining how far the surrogate loss $L(w, \sS, \sTa)$ is from the CRF loss $L(w, \sS)$.
To bound the approximation error, we make the following assumption about the proposal distributions $\pR(x,w)$.
\begin{assumption}
\label{ass:distribution}
For all $(x_i, y_i) \in \sS$ and weight vectors $w \in \dW$, the proposal distribution 
satisfies the following condition with probability at least $1 - \nicefrac{\NormI{w}}{\sqrt{m}}$, for
a constant $c \in [0,1]$:
\begin{enumerate}[(i)]
\item $\sT_i = \Set{y_i}$ if $\forall\; y \neq y_i \Inner{\phi(x_i, y_i)}{w} > \Inner{\phi(x_i, y)}{w}$,
\item $\frac{1}{n} \sum_{y\in \sT_i} \Inner{\phi(x_i,y)}{w} \geq \Inner{\phi(x_i,y_i)}{w} + 
	c\NormI{w}$ otherwise, 
\end{enumerate}
where the probability is taken over the set $\sT_i$.
\end{assumption}
Intuitively, Assumption \ref{ass:distribution} states that, if $y_i$ is not the highest scoring structure under $w$, 
then the proposal distribution should return structures $\sT = \Set{y}$ whose average score is an 
additive constant factor away from the score of the observed structure $y_i$ with high probability.
Otherwise, the proposal distribution should return the 
singleton set $\sT = \Set{y_i}$ with high probability.
Note that Assumption \ref{ass:distribution} is in comparison much weaker than the low-norm assumption of \cite{honorio_structured_2015},
which requires that, in expectation, the norm of the difference between $\phi(x,y)$ and $\phi(x,y_i)$ (where $y$
is sampled from the proposal distribution) should decay as $\nicefrac{1}{\sqrt{m}}$. The following lemma
bounds the approximation error from above.

\begin{lemma}[Approximation Error]
\label{lemma:approx_error}
If the scale parameter of the Gumbel perturbations satisfies: $\beta \leq \min(\nicefrac{\NormI{w}}{\log m}, \nicefrac{\wmin}{\log((r-1)(\sqrt{m} - 1))})$ for all $w \neq 0$, and $n \geq m^{0.5 - c}$, then under Assumption \ref{ass:distribution}
$A(w, \sS) \leq \varepsilon_1(m, n, w)$, where
\begin{equation*}
\varepsilon_1(m, n, w) \defeq \frac{\NormI{w}}{\sqrt{m}} + \frac{1}{1 + \sqrt{m}},
\end{equation*}
and $\wmin = \min\Set{\Abs{w_j} \mid \Abs{w_j} \neq 0, j \in [d]}$.
\end{lemma}
\begin{proof}
Let $A_i(w,\sS) \defeq \Prob[\gamma \sim \pG(\beta)]{f_{w,\gamma}(x_i) \neq y_i} - 
	\Exp[\sT_i]{\Prob[\gamma \sim \pG(\beta)]{f_{w,\gamma,\sTa_i}(x_i) \neq y_i}}$
be the i-th term of $A(w, \sS)$. We will consider two cases. 
\paragraph{Case I:} $y_i$ is strictly the highest scoring structure for $x_i$ under $w$, i.e.,
$\forall y \neq y_i \, \Inner{\phi(x_i, y_i)}{w} > \Inner{\phi(x_i, y)}{w}$. First note that:
\begin{equation}
A_i(w, \sS) \leq \Prob[\gamma \sim \pG(\beta)]{f_{w,\gamma}(x_i) \neq y_i}. \label{eq:proof_app_err_1}
\end{equation}
We will prove that
$\Prob[\gamma \sim \pG(\beta)]{f_{w,\gamma}(x_i) \neq y_i} \leq \nicefrac{1}{\sqrt{m}}$. Assume instead that  
$\Prob[\gamma \sim \pG(\beta)]{f_{w,\gamma}(x_i) \neq y_i} > \nicefrac{1}{\sqrt{m}}$.
Then
\begin{gather*} 
\sum_{y \neq y_i} (\sqrt{m} - 1) e^{\Inner{\phi(x_i, y)}{w}/\beta} > e^{\Inner{\phi(x_i, y_i)}{w}/\beta}
\end{gather*}
Let $y' \in \dY(x_i) \setminus \Set{y_i}$ be such that $\Inner{\phi(x_i, y')}{w}$ is maximized. Then, 
$(r - 1)(\sqrt{m} - 1) e^{\Inner{\phi(x_i, y')}{w}/\beta}$ upper bounds the left-hand side of the above equation.
Taking log on both sides we get:
\begin{gather*}
 \beta > \dfrac{\Inner{\phi(x_i,y_i) - \phi(x_i, y')}{w}}{\log((r - 1) (\sqrt{m} - 1))}
\end{gather*}
Since $y_i$ is the unique maximizer of the score $\Inner{\phi(x_i, y_i)}{w}$,
$\phi(x_i, y')$ and $\phi(x_i, y_i)$ must differ on at least one element in the support set of $w$.
This implies, from above and the assumption that the minimum non-zero element of $\phi(x,y)$ is at least 1:
\begin{equation*}
\beta > \dfrac{\wmin}{\log((r - 1) (\sqrt{m} - 1))},
\end{equation*}
which violates Assumption \ref{ass:distribution}. Therefore from \eqref{eq:proof_app_err_1} we have that
$A_i(w, \sS) \leq \nicefrac{1}{\sqrt{m}}$.

\paragraph{Case II:} $\exists y \neq y_i \,: \Inner{\phi(x_i, y)}{w} \geq \Inner{\phi(x_i, y_i)}{w}$.
Let $\Delta_i(y) \defeq \phi(x_i, y) - \phi(x_i, y_i)$. In this case,
\begin{align}
A_i(w, \sS) &\Ineq[(a)]{\leq} \Exp[\sT_i]{\Prob[\gamma]{f_{w, \gamma, \sTa_i}(x_i) = y_i}} \notag \\
	&=  \Exp[\sT_i]{ 
			\frac{\exp(\Inner{\phi(x_i, y_i)}{w}/\beta)}{Z(w,x_i,\sTa_i)}  
		} 		\notag \\
	 &\Ineq[(b)]{=} \Exp[\sT_i]{ 
			\frac{1}{1 + \sum_{y \in \sT_i} e^{\Inner{\Delta_i(y)}{w}/\beta}}
		} \notag \\
	&\Ineq[(c)]{\leq} \Exp[S_i]{\frac{1}{1 + n e^{S_i/\beta}}}, \label{eq_proof_app_err_2}
\end{align}
where we have defined $S_i \defeq \frac{1}{n} \sum_{y \in \sT_i} \Inner{\Delta_i(y)}{w}$. 
In the above, in step (a) we dropped the term $\Prob[\gamma]{f_{w,\gamma}(x_i) = y_i}$ to get an upper bound.
Step (b) follows from dividing the numerator and denominator by $\exp(\Inner{\phi(x_i, y_i)}{w})$
and that $y_i \in \sTa_i$. Step (c) follows from Jensen's inequality. Now,
\begin{align}
\Exp[S_i]{\frac{1}{1 + n e^{S_i/\beta}}} 
&= \Exp[S_i]{\frac{1}{1 + n e^{S_i/\beta}} \mid S_i \geq \frac{\NormI{w}}{2}} \Prob{S_i \geq \frac{\NormI{w}}{2}} \notag \\
 &\quad + \Exp[S_i]{\frac{1}{1 + n e^{S_i/\beta}} \mid S_i < \frac{\NormI{w}}{2}} \Prob{S_i < \frac{\NormI{w}}{2}}  \notag \\
&\Ineq[(a)]{\leq}  \Exp[S_i]{\frac{1}{1 + n e^{S_i/\beta}} \mid S_i \geq \frac{\NormI{w}}{2}} + \frac{\NormI{w}}{\sqrt{m}} \notag \\
&\Ineq[(b)]{\leq} \Exp[S_i]{\frac{1}{1 + n e^{\nicefrac{S_i \log m}{\NormI{w}}}} \mid S_i \geq \frac{\NormI{w}}{2}} 
	+ \frac{\NormI{w}}{\sqrt{m}} \notag \\
&= \Exp[S_i]{\frac{1}{1 + n m^{\nicefrac{S_i}{\NormI{w}}}} \mid S_i \geq \frac{\NormI{w}}{2}} + \frac{\NormI{w}}{\sqrt{m}} \notag \\
&\leq \frac{1}{1 + n \sqrt{m}} + \frac{\NormI{w}}{\sqrt{m}}, \label{eq_proof_app_err_3}
\end{align}
where inequality (a) follows from Assumption \ref{ass:distribution} and (b) follows from the fact that
 $\beta \leq \nicefrac{\NormI{w}}{\log m}$.
Thus from \eqref{eq_proof_app_err_2} and \eqref{eq_proof_app_err_3} we have that 
$A_i(w, \sS) \leq \nicefrac{1}{(1 + n \sqrt{m})} + \nicefrac{\NormI{w}}{\sqrt{m}}$.

The final claim follows from Case I and II.
\end{proof}
Note that for $c \geq 0.5$ the number of structured outputs needed is $n = 1$, while in the worst case ($c = 0$)
$n = \sqrt{m}$. Furthermore, $n$ needs to grow
polynomially with respect to $m$ in order to achieve $\BigO{\nicefrac{1}{\sqrt{m}}}$ generalization error.
\begin{lemma}[Statistical Error]
\label{lemma:stat_error}
For any fixed data set $\sS$, the statistical error $B(w, \sS, \sTa)$ is bounded, simultaneously for
all proposal distributions $\pR(x_i, w)$ over $\Set{\sT_i}$, as follows:
\begin{align}
\Prob[\sT]{ (\forall w \in \dW)\; B(w, \sS, \sTa)  \leq
\varepsilon_2(d,s,n,r,m,\delta) \mid \sS}  \geq 1 - \delta,
\end{align}
where
\begin{align*}
\varepsilon_2(d,s,n,r,m,\delta) &\defeq 2 \sqrt{\frac{s (\ln d + 2 \ln (nr))}{m}} + \sqrt{\frac{\ln \nicefrac{1}{\delta}}{2m}} + 
 \sqrt{\frac{s (\ln d + 2 \ln (mr)) + \ln \nicefrac{1}{\delta}}{2m}}.
\end{align*}
\end{lemma}
\begin{proof}
We adapt the proof of Rademacher based uniform convergence for our purpose.
Fix the distribution over $\sT$ to $\pR(\sS, w')$ for some $w'$.
Recall that $\sTa = \Set{\sTa_i}$ with $\sTa_i = \Set{y_i} \union \sT_i$ and
the elements of $\sT_i$ are drawn \iid from $\pR(x_i, w')$.
Since the only random part in $\sTa_i$ is $\sT_i$ and $y_i \in \sS$, it suffices to show concentration
of $\Exp[\sT]{L(w, \sS, \sT)} -  L(w, \sS, \sT)$ for all $w$ and $\sS$.
For a fixed $\sS$, we
will consider $L(w, \sS, \sT)$ to be a function of $\sT$ and $w$ and denote it by $L(\sT, w; \sS)$. In what follows,
we will consider $\sT$ to be an $mn$-dimensional vector whose elements (structured outputs) are conditionally independent
(but not identically distributed) given a data set $\sS$.
Define, 
\begin{align}
\varphi(\sT; \sS) \defeq \sup_{w \in \dW} \Exp[\sT \sim \pR(\sS,w')]{L(\sT, w; \sS)}  - L(\sT, w; \sS).
\end{align}
$\varphi(\sT; \sS)$ is $(\nicefrac{1}{m})$-Lipschitz and the elements of $\sT$ are independent. Therefore, by McDiarmid's inequality,
we have that:
\begin{align}
\Prob[\sT]{\Exp[\sT]{\varphi(\sT; \sS)} - \varphi(\sT; \sS) \leq \sqrt{\frac{\ln (\nicefrac{1}{\delta})}{2m}} \mid \sS} \geq 1 - \delta.
\end{align}
Therefore, with probability at least $1 - \delta$ over the choice of $\sT$:
\begin{align}
(\forall w \in \dW) \; \Exp[\sT]{L(\sT, w; \sS)} - L(\sT, w; \sS)  
& \leq \sup_{w \in \dW} \Exp[\sT]{L(\sT, w; \sS)} - L(\sT, w; \sS) \notag \\
	&= \varphi(\sT; \sS) 
	 \leq \Exp[\sT]{\varphi(\sT; \sS)} + \sqrt{\frac{\ln \nicefrac{1}{\delta}}{2m}} \label{eq:st_error_1}.
\end{align}
Next, we will use a symmetrization argument to bound $\Exp[\sT]{\varphi(\sT; \sS)}$. 
Let $\sT' \sim \pR(\sS)$ be an independent copy of $\sT$.
Observe that:
\begin{align*}
\Exp[\sT']{L(\sT, w; \sS) \mid \sT} &= L(\sT, w; \sS) \\
\Exp[\sT']{L(\sT', w; \sS) \mid \sT} &= \Exp[\sT]{L(\sT, w; \sS)}.
\end{align*}
Now,
\begin{align*}
\Exp[\sT]{\varphi(\sT)} 
& = \Exp[\sT]{\! \sup_{w \in \dW}\! \Exp[\sT]{L(\sT, w; \sS)} - L(\sT, w; \sS) } \\
& = \Exp[\sT]{\! \sup_{w \in \dW}\! \Exp[\sT']{L(\sT', w; \sS) \mid \sT} - \Exp[\sT']{L(\sT, w; \sS) \mid \sT}  } \\
& \leq \Exp[\sT, \sT']{ \sup_{w \in \dW}\! \frac{1}{m} \sum_{i=1}^m z_i' - z_i },
\end{align*}
where $z_i' = \Prob[\gamma]{f_{w,\gamma,\sT'}(x_i) \neq y_i}$ and $z_i = \Prob[\gamma]{f_{w,\gamma,\sT}(x_i) \neq y_i}$.
Since $z_i' - z_i$ has a distribution that is symmetric around zero, $z_i' - z_i$ and $\sigma_i(z_i' - z_i)$
have the same distribution, where $\sigma_i$'s are independent Rademacher variables. Continuing the above derivation,
\begin{align*}
\Exp[\sT]{\varphi(\sT)} 
	& \leq \Exp[\sT, \sT', \sigma]{ \sup_{w \in \dW} \frac{1}{m} \sum_{i=1}^m \sigma_i(z_i' - z_i) } \\
	& = \frac{2}{m} \Exp[\sT, \sigma]{ \sup_{w \in \dW} \sum_{i=1}^m \sigma_i \Prob[\gamma]{f_{w,\gamma,\sT}(x_i) \neq y_i}} \\
	& = 2\Exp[\sT]{\eRC_{\sT}(\mathcal{G})},
\end{align*}
where $\eRC_{\sT}(\mathcal{G})$ is the empirical Rademacher complexity of the function class $\mathcal{G} = \Set{g_{w} \mid w \in \dW}$
with respect to $\sT$, with $g_w(x,y) = \Prob[\gamma]{f_{w,\gamma,\sT}(x) \neq y}$.
Next, using the same argument as in the proof of Theorem \ref{thm:gen_bound}, we can bound $\eRC_{\sT}(\mathcal{G})$ 
for any set $\sT$, and get the following bound:
\begin{align}
\Exp[\sT]{\varphi(\sT)} \leq 
2 \sqrt{\frac{s (\log d + 2 \log (nr))}{m}} \label{eq:st_error_2}
\end{align}
Note that the above differs from the bound in Theorem \ref{thm:gen_bound} in the log factor since we need to
consider linear orderings of $nr$ structured outputs.
Therefore from \eqref{eq:st_error_1} and \eqref{eq:st_error_2} we have that:
\begin{align}
\mathrm{Pr}_{\sT}\{ (\forall w \in \dW)\; \Exp[\sT]{L(\sT, w; \sS)} - L(\sT, w; \sS) 
\leq \varepsilon_2(d,s,n,r,m,\delta) \mid \sS \} \geq 1 - \delta.
\end{align}
By Definition \ref{def:equivalent_proposals}
and from the results by \cite{bennett1956determination,
bennett1960multidimensional, cover1967number},
there are at most ${d \choose s} (mr)^{2s}$ effective
(equivalence classes) proposal distributions $\pR(.)$
Taking a union bound over all such proposal distributions we prove our claim.
\end{proof}
Now, we are ready to present our main result proving uniform convergence of the randomized loss $L(w, \sS, \sTa)$.
More specifically, we provide $\BigOT{\nicefrac{1}{\sqrt{m}}}$ generalization error. 
\begin{theorem}
With probability at least $1 - 2\delta$ over the choice of the data set $\sS$ and the set of random structured outputs $\sT$,
and simultaneously for all $w \in \dW$ and proposal distributions $\pR(x,w)$:
\begin{equation}
L(w, \pD) \leq L(w, \sS, \sTa) + \varepsilon_1 + \varepsilon_2,
\end{equation}
where $\varepsilon_1$ and $\varepsilon_2$ are defined in Lemma \ref{lemma:approx_error} and \ref{lemma:stat_error} respectively.
\end{theorem}
\begin{proof}
The claim follows directly from Lemma \ref{lemma:approx_error} and Lemma \ref{lemma:stat_error} 
by taking an expectation with respect to $\sS$.
\end{proof}

\subsection{Examples of proposal distributions}
Having proved uniform convergence of our randomized procedure for learning the parameters of a MAP decoder,
we turn our attention to the proposal distribution. We want to construct proposal distributions of the form
given by Definition \ref{def:equivalent_proposals} that satisfy Assumption \ref{ass:distribution} with 
a large enough constant $c$. Additionally, for our randomized procedure to run in polynomial time we want the proposal
distribution to sample a structured output in constant time. 
The following algorithm is directly motivated  by Assumption \ref{ass:distribution} where the
set $\text{neighbors}_k(y)$ for an input $x$ is defined as: 
$\text{neighbors}_k(y) \defeq \Set{y' \in \dY(x) \setminus \Set{y} \mid H(y,y') \leq k}$,
with $H(\parg, \parg)$ being the Hamming distance.
\begin{algorithm}[H]
\begin{algorithmic}[1]
\caption{An example algorithm implementing a proposal distribution that depends on $y_i \in \sS$. \label{alg:proposal1}}
\Require Weight vector $w \in \dW$, $(x_i, y_i) \in \sS$, parameter $\alpha \in [0,1]$ and $k \geq 1$.
\Ensure A structured output $y \in \dY(x)$.
\State With probability  $\alpha$ pick $y'$ uniformly at random from $\dY(x_i)$,
and with probability $1 - \alpha$ set $y'$ to $y_i$.
\State $y \gets y'$.
\For{$y' \in \textsf{neighbors}_k(y)$} 
\If{$\Inner{\phi(x,y')}{w} \geq \Inner{\phi(x,y)}{w}$}
\State $y \gets y'$.
\EndIf
\EndFor
\State \Return $y$.
\end{algorithmic}
\end{algorithm}
\begin{remark}
Setting $\alpha = \nicefrac{\NormI{w}}{\sqrt{m}}$,
Algorithm \ref{alg:proposal1} satisfies the condition given in Definition \ref{def:equivalent_proposals}
as well as Assumption \ref{ass:distribution}. Since,
for any $w,w' \in \dW$ that induce the same linear ordering over $\dY(x)$, conditioned on the $y'$ sampled
in Step 3, the algorithm returns the same $y$ for both $w$ and $w'$ with probability $1$.
\end{remark}
Also note that using a larger $k$ ensures that the above algorithm satisfies Assumption \ref{ass:distribution}
with a larger constant $c$, thereby reducing the number of structured outputs that need to be sampled ($n$),
at the cost of increased computation for sampling a single structured output.

The parameter $\alpha$ in Algorithm \ref{alg:proposal1} controls exploration vs exploitation.
As $\alpha$ becomes smaller Algorithm \ref{alg:proposal1} returns a proposal from within
the neighborhood of $y_i$ while for larger $\alpha$ it explores high scoring structures in 
the entire set of candidate structures.

\subsection{Minimizing the CRF loss}
In this section we discuss strategies for minimizing the (randomized) CRF loss $L(w, \sS, \sTa)$.
Minimizing the randomized CRF loss $L(w, \sS, \sTa)$ is equivalent to maximizing the \emph{randomized CRF gain} 
$U(w, \sS, \sTa) = \frac{1}{m} \sum_{i=1}^m \Prob[\gamma]{f_{w, \gamma, \sTa_i}(x_i) = y_i}$,
which in turn is equivalent to maximizing $\log U(w, \sS, \sTa)$. The latter can be accomplished
by gradient based methods with the gradient of $\log U(w, \sS, \sTa)$ given by:
\begin{equation}
\Grad_{w} \log U(w, \sS, \sTa) = \frac{\sum_{i=1}^m q_i (\phi(x_i, y_i) - \Exp{\phi(x_i, y)})}{\sum_{i=1}^m q_i},
\label{eq:grad}
\end{equation}
where $q_i \defeq \Prob[\gamma]{f_{w, \gamma, \sTa_i}(x_i) = y_i}$, and the expectation is taken with respect to
$y \sim \pQ(x_i, w, \sTa_i)$. The exact CRF loss ($L(w, \sS)$) can similarly be minimized
by using $\sTa_i = \dY(x_i)$, for all $x_i \in \sS$, in the above. Note that by Jensen's inequality
$\log U(w, \sS, \sTa) \geq \frac{1}{m} \sum_{i=1}^m \log \Prob[\gamma]{f_{w, \gamma, \sTa_i}(x_i) = y_i}$,
where the latter can be identified as the log likelihood of the data set $\sS$ under the CRF distributions $\Set{\pQ(x_i, w, \sTa_i)}$.
Therefore, $L(w, \sS, \sTa)$ can be equivalently minimized by minimizing the negative log-likelihood of the data, which in turn
gives rise to the well known \emph{moment-matching} rule known in the literature. Thus, Algorithm \ref{alg:proposal1} 
can be used with standard moment matching where the expectation is approximated by averaging over
$y$'s drawn from the distribution $\pQ(x_i, w, \sTa_i)$. While standard moment matching is in general intractable,
moment matching in conjunction with Algorithm \ref{alg:proposal1} is always efficient. 
Indeed, \eqref{eq:grad} can be thought of as a ``weighted'' moment matching rule with weights $q_i$.

\begin{figure*}[htbp]
\fbox{%
\begin{minipage}{0.99\linewidth}
\centering
\begin{tabular}{p{0.48\linewidth} p{0.48\linewidth}}
\vspace{0pt} \includegraphics[width=\linewidth]{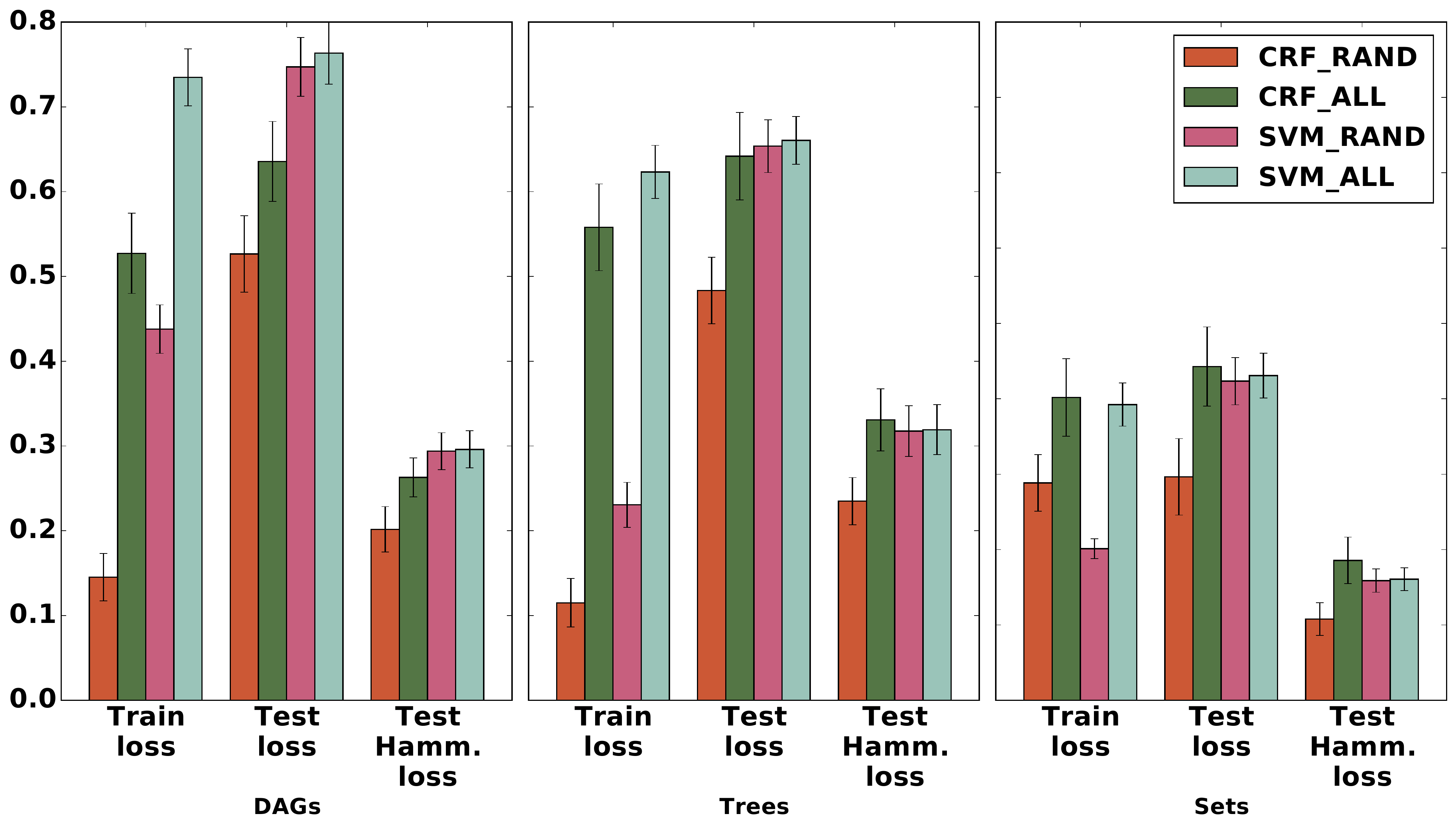} &
\vspace{0pt} \includegraphics[width=\linewidth]{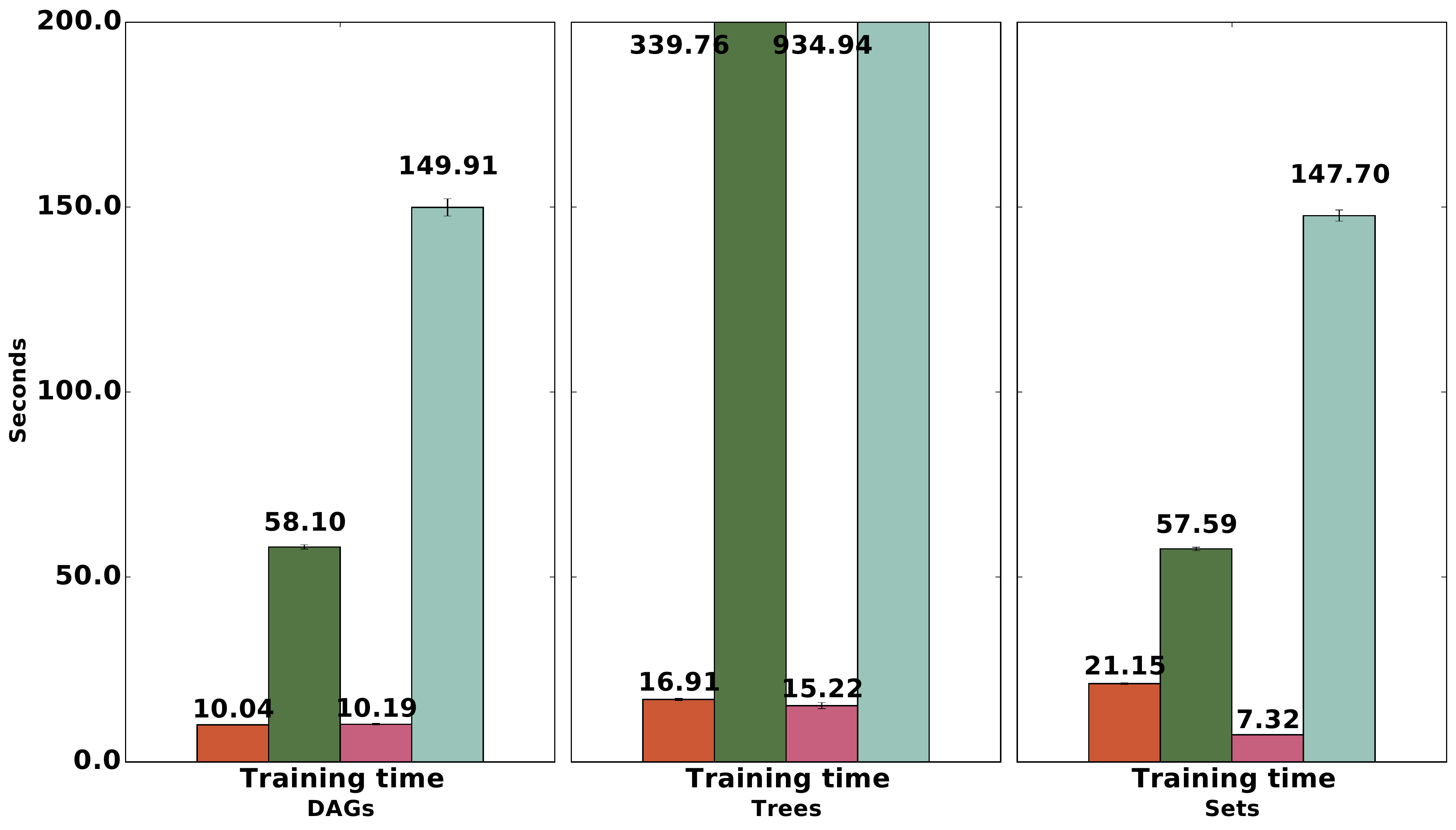}
\end{tabular}
\caption{(\textbf{Left}) Training and test set loss \eqref{eq:full_loss}, and test set hamming loss
of the exact method (\texttt{CRF\_ALL}) and our randomized algorithm (\texttt{CRF\_RAND}), the randomized SVM
method by \cite{honorio_structured_2015} (\texttt{SVM\_RAND}), and the exact SVM (\texttt{SVM\_ALL}), a.k.a max-margin, method 
of \cite{taskar_max-margin_2003}. For the randomized algorithms, i.e., \CRFRAND{} and \SVMRAND{}, the training loss is the
randomized training loss, i.e., $L(w, \sS, \sTa)$ and $L(w, \sS, \sT)$ respectively.
(\textbf{Right}) Training time in seconds for the various methods.
\label{fig:error}}
\end{minipage}}
\end{figure*}

\section{Experiments}
In this section, we evaluate our proposed method (\CRFRAND{}) on synthetic data
against three other methods: \CRFALL{}, \SVMRAND{}, and \SVM{}. The \CRFRAND{} method
minimizes the randomized loss $L(w, \sS, \sTa)$ \eqref{eq:rand_loss} subject to $\ell_1$ penalty
(as prescribed by Lemma \ref{lemma:approx_error}) by sampling
structured outputs from the proposal distribution given by Algorithm \ref{alg:proposal1}.
The \CRFALL{} method minimizes the exact (exponential-time) loss $L(w, \sS)$ \eqref{eq:full_loss}.
Lastly, \SVM{} is the widely used max-margin method of \cite{taskar_max-margin_2003},
while \SVMRAND{} is the randomized SVM method proposed by \cite{honorio_structured_2015}.

We generate a ground truth parameter $w^* \in \R^d$ with random 
entries sampled independently from a zero mean Gaussian distribution
with variance $100$. We then randomly set all but $s = \sqrt{d}$ entries to be zero.
We then generate a training set of $S$ of 100 samples.
We used the following joint feature map $\phi(x,y)$ for an input output pair.
For every pair of possible edges or elements $i$ and $j$, we set 
$\phi(x,y)_{i,j} = \Ind{x_{i,j} = 1 \Land i \in y \Land j \in y}$.
For instance, for directed spanning trees of $v$ nodes, we have ${x \in \{0,1\}^{\binom{v}{2}}}$ 
and ${\phi(x,y) \in \R^{\binom{v}{2}}}$. 
We considered directed spanning trees of $6$ nodes, 
directed acyclic graphs of $5$ nodes and $2$ parents per node, 
and sets of $4$ elements chosen from $15$ possible elements.
In order to generate each training sample ${(x,y) \in S}$, 
we generated a random vector $x$ with independent Bernoulli entries with parameter $\nicefrac{1}{2}$.
After generating $x$, we set ${y = f_{w^*}(x)}$, i.e., 
we solved \eqref{eq:decoder} in order to produce the latent structured output $y$ 
from the observed input $x$ and the parameter ${w^*}$.

We set the $\ell_1$ regularization parameter to be $0.01$ for all methods.
We used 20 iterations of gradient descent with step size of $\nicefrac{1}{\sqrt{t}}$ for all algorithms, where $t$ is the 
iteration, to learn the parameter $w$ for both the exact method and our randomized algorithm.
In order to simplify gradient calculations, we simply set $\beta = 1/\log((r - 1))(\sqrt{m} - 1))$ during training.
For \CRFRAND, we used Algorithm \ref{alg:proposal1} with $\alpha = \nicefrac{\NormI{w}}{\sqrt{m}}$ and invoke the 
algorithm $\sqrt{m}$ number of times to generate the set $\sT_i$ for each $i \in [m]$ and $w$.
This results in $n = \Abs{\sT_i} \leq \sqrt{m}$.
To evaluate the generalization performance of our algorithm we generated a test set $S' = \Set{x'_i, y'_i}_{i=1}^m$
of 100 samples and calculated two losses. The first was the full CRF loss \eqref{eq:full_loss} on the test set $S'$,
and the second was the test set hamming loss $\frac{1}{m} \sum_{i=1}^m \hat{H}(f_{\hat{w}}(x_i'), y_i')$, 
where $\hat{H}(\parg, \parg)$ is the normalized Hamming distance, and $\hat{w}$ is the learned parameter.
Hamming distance is a popular distortion function
used in structured prediction, and provides a more realistic assessment of the performance of a decoder,
since in most cases it suffices to recover most of the structure rather than predicting the structure exactly.
For DAGs and trees the Hamming distance counts the number of different edges between the structured outputs,
while for sets it counts the number of different elements. We normalize the Hamming distance to be between 0 and 1.
We computed the mean and $95\%$ confidence intervals of each of these metrics by repeating the above procedure 30 times.

Figure \ref{fig:error} shows the training and test set errors and the training time 
of the four different algorithms. \CRFRAND{} significantly outperformed other algorithms in both the test set loss
and test set hamming loss, while being $\approx$ 6 times faster than the exact method (\CRFALL{}) for DAGs,
$\approx$ 20 times faster for trees, and $\approx$ 3 times faster for sets. The exact CRF method (\CRFALL{}) was also
significantly faster than the exact SVM (\SVM{}) method while achieving similar test set loss and test set hamming loss.
\section{Related Work}
\label{sec:related_work}
Significant body of work exists in computing a single MAP estimate by exploiting problem 
specific structure, for instance, super-modularity, linear programming relaxations 
to name a few. However, in this paper we are concerned with the problem of learning
the parameters of MAP perturbation models. Among generalization bounds for MAP perturbation
models, \cite{hazan2013learning} prove PAC-Bayesian generalization bounds for weight based perturbations.
\cite{hazan2013learning} additionally propose learning weight based MAP-perturbation models by 
minimizing the PAC-Bayesian upper bound on the generalization error. However, 
their method for learning the parameters involves constructing restricted families of posterior distributions over 
the weights $w$ that lead to smooth, but not necessarily convex, generalization bounds that can be optimized using
gradient based methods. For learning MAP-perturbation models with structure based (Gumbel) perturbations,
\cite{gane2014learning} propose a hard-EM algorithm which is both worst-case exponential time
and has no theoretical guarantees. \cite{papandreou2011perturb} on the other hand, propose learning
Gumbel MAP-perturbation models by using the moment matching method. However, such an approach is
tractable only for energy functions for which the global minimum can be computed efficiently.
Lastly, \cite{hazan2013sampling,orabona2014measure} consider the problem of 
efficiently sampling from MAP perturbation models using low dimensional perturbations.
\cite{hazan2012partition, hazan2013sampling} additionally propose ways to approximate 
and bound the partition function. While such bounds on the partition function can be used, in principle,
to approximately minimize the CRF loss \eqref{eq:full_loss}, it is unclear if one can obtain uniform
convergence guarantees for the same, given that computing or even approximating the partition function
is NP-hard \cite{barahona1982computational, chandrasekaran2008complexity}.
\section{Concluding remarks}
We conclude with some directions for future work.
While in this work we showed that one can learn with approximate inference,
it would be interesting to analyze approximate inference for prediction on an independent test set.
Another avenue for future work would be to develop more powerful proposal distributions that allow
for more finer-grained control over the parameter $c$ by exploiting problem specific structure like submodularity.

\section*{Acknowledgements}
This material is based upon work supported by the National Science Foundation under Grant No. 1716609-IIS.

\bibliographystyle{apalike}
\bibliography{paper}

\begin{thebibliography}{18}
\providecommand{\natexlab}[1]{#1}
\providecommand{\url}[1]{\texttt{#1}}
\expandafter\ifx\csname urlstyle\endcsname\relax
  \providecommand{\doi}[1]{doi: #1}\else
  \providecommand{\doi}{doi: \begingroup \urlstyle{rm}\Url}\fi

\bibitem[Barahona(1982)]{barahona1982computational}
Barahona, Francisco.
\newblock On the computational complexity of ising spin glass models.
\newblock \emph{Journal of Physics A: Mathematical and General}, 15\penalty0
  (10):\penalty0 3241, 1982.

\bibitem[Bennett(1956)]{bennett1956determination}
Bennett, Joseph~F.
\newblock Determination of the number of independent parameters of a score
  matrix from the examination of rank orders.
\newblock \emph{Psychometrika}, 21\penalty0 (4):\penalty0 383--393, 1956.

\bibitem[Bennett \& Hays(1960)Bennett and Hays]{bennett1960multidimensional}
Bennett, Joseph~F and Hays, William~L.
\newblock Multidimensional unfolding: Determining the dimensionality of ranked
  preference data.
\newblock \emph{Psychometrika}, 25\penalty0 (1):\penalty0 27--43, 1960.

\bibitem[Chandrasekaran et~al.(2008)Chandrasekaran, Srebro, and
  Harsha]{chandrasekaran2008complexity}
Chandrasekaran, Venkat, Srebro, Nathan, and Harsha, Prahladh.
\newblock Complexity of inference in graphical models.
\newblock In \emph{Proceedings of the Twenty-Fourth Conference on Uncertainty
  in Artificial Intelligence}, pp.\  70--78. AUAI Press, 2008.

\bibitem[Cover(1967)]{cover1967number}
Cover, Thomas~M.
\newblock The number of linearly inducible orderings of points in d-space.
\newblock \emph{SIAM Journal on Applied Mathematics}, 15\penalty0 (2):\penalty0
  434--439, 1967.

\bibitem[Gane et~al.(2014)Gane, Hazan, and Jaakkola]{gane2014learning}
Gane, Andreea, Hazan, Tamir, and Jaakkola, Tommi.
\newblock Learning with maximum a-posteriori perturbation models.
\newblock In \emph{Artificial Intelligence and Statistics}, pp.\  247--256,
  2014.

\bibitem[Gumbel(1954)]{gumbel1954statistical}
Gumbel, Emil~Julius.
\newblock Statistical theory of extreme valuse and some practical applications.
\newblock \emph{Nat. Bur. Standards Appl. Math. Ser. 33}, 1954.

\bibitem[Guruswami et~al.(2008)Guruswami, Manokaran, and
  Raghavendra]{guruswami2008beating}
Guruswami, Venkatesan, Manokaran, Rajsekar, and Raghavendra, Prasad.
\newblock Beating the random ordering is hard: Inapproximability of maximum
  acyclic subgraph.
\newblock In \emph{Foundations of Computer Science, 2008. FOCS'08. IEEE 49th
  Annual IEEE Symposium on}, pp.\  573--582. IEEE, 2008.

\bibitem[Hazan \& Jaakkola(2012)Hazan and Jaakkola]{hazan2012partition}
Hazan, Tamir and Jaakkola, Tommi.
\newblock On the partition function and random maximum a-posteriori
  perturbations.
\newblock In \emph{Proceedings of the 29th International Coference on
  International Conference on Machine Learning}, pp.\  1667--1674. Omnipress,
  2012.

\bibitem[Hazan et~al.(2013{\natexlab{a}})Hazan, Maji, and
  Jaakkola]{hazan2013sampling}
Hazan, Tamir, Maji, Subhransu, and Jaakkola, Tommi.
\newblock On sampling from the gibbs distribution with random maximum
  a-posteriori perturbations.
\newblock In \emph{Advances in Neural Information Processing Systems}, pp.\
  1268--1276, 2013{\natexlab{a}}.

\bibitem[Hazan et~al.(2013{\natexlab{b}})Hazan, Maji, Keshet, and
  Jaakkola]{hazan2013learning}
Hazan, Tamir, Maji, Subhransu, Keshet, Joseph, and Jaakkola, Tommi.
\newblock Learning efficient random maximum a-posteriori predictors with
  non-decomposable loss functions.
\newblock In \emph{Advances in Neural Information Processing Systems}, pp.\
  1887--1895, 2013{\natexlab{b}}.

\bibitem[Honorio \& Jaakkola(2016)Honorio and
  Jaakkola]{honorio_structured_2015}
Honorio, Jean and Jaakkola, Tommi.
\newblock Structured prediction: from gaussian perturbations to linear-time
  principled algorithms.
\newblock In \emph{Proceedings of the Thirty-Second Conference on Uncertainty
  in Artificial Intelligence}, pp.\  271--278. AUAI Press, 2016.

\bibitem[Lafferty et~al.(2001)Lafferty, McCallum, and
  Pereira]{lafferty2001conditional}
Lafferty, John~D., McCallum, Andrew, and Pereira, Fernando C.~N.
\newblock Conditional random fields: Probabilistic models for segmenting and
  labeling sequence data.
\newblock In \emph{Proceedings of the Eighteenth International Conference on
  Machine Learning}, ICML '01, pp.\  282--289, San Francisco, CA, USA, 2001.
  Morgan Kaufmann Publishers Inc.
\newblock ISBN 1-55860-778-1.
\newblock URL \url{http://dl.acm.org/citation.cfm?id=645530.655813}.

\bibitem[McAllester(2007)]{mcallester_generalization_2007}
McAllester, David.
\newblock Generalization bounds and consistency.
\newblock \emph{Predicting structured data}, pp.\  247--261, 2007.
\newblock URL \url{http://nagoya.uchicago.edu/~dmcallester/colbounds.pdf}.

\bibitem[Nemhauser et~al.(1978)Nemhauser, Wolsey, and
  Fisher]{nemhauser1978analysis}
Nemhauser, George~L, Wolsey, Laurence~A, and Fisher, Marshall~L.
\newblock An analysis of approximations for maximizing submodular set
  functions—{I}.
\newblock \emph{Mathematical Programming}, 14\penalty0 (1):\penalty0 265--294,
  1978.

\bibitem[Orabona et~al.(2014)Orabona, Hazan, Sarwate, and
  Jaakkola]{orabona2014measure}
Orabona, Francesco, Hazan, Tamir, Sarwate, Anand, and Jaakkola, Tommi.
\newblock On measure concentration of random maximum a-posteriori
  perturbations.
\newblock In \emph{International Conference on Machine Learning}, pp.\
  432--440, 2014.

\bibitem[Papandreou \& Yuille(2011)Papandreou and
  Yuille]{papandreou2011perturb}
Papandreou, George and Yuille, Alan~L.
\newblock Perturb-and-map random fields: Using discrete optimization to learn
  and sample from energy models.
\newblock In \emph{Computer Vision (ICCV), 2011 IEEE International Conference
  on}, pp.\  193--200. IEEE, 2011.

\bibitem[Taskar et~al.(2003)Taskar, Guestrin, and
  Koller]{taskar_max-margin_2003}
Taskar, Ben, Guestrin, Carlos, and Koller, Daphne.
\newblock Max-margin {Markov} {Networks}.
\newblock In \emph{Proceedings of the 16th {International} {Conference} on
  {Neural} {Information} {Processing} {Systems}}, {NIPS}'03, pp.\  25--32,
  Cambridge, MA, USA, 2003. MIT Press.
\newblock URL \url{http://dl.acm.org/citation.cfm?id=2981345.2981349}.

\end{thebibliography}

\end{document}